\documentclass{article}

\usepackage[nonatbib,final]{neurips_2021}

\usepackage{xcolor}
\usepackage{hyperref}
\hypersetup{
    colorlinks,
    linkcolor={red!50!black},
    citecolor={green!50!black},
    urlcolor={blue!80!black}
}
\usepackage[linesnumbered, ruled, vlined]{algorithm2e}
\usepackage{amsfonts}
\usepackage{amsmath}
\usepackage{amssymb}
\usepackage{amsthm}
\usepackage{bm}
\usepackage{breakcites}
\usepackage[capitalize]{cleveref}
\usepackage{dsfont}
\usepackage[inline, shortlabels]{enumitem}
\usepackage[T1]{fontenc}    
\usepackage[utf8]{inputenc}
\usepackage{geometry}
\usepackage{mathtools}
\usepackage{microtype}
\usepackage{nicefrac}
\usepackage{scalerel}
\usepackage{parskip}
\usepackage{textcomp}
\usepackage{tikz-cd}
\usepackage{url}

\usepackage[backend=bibtex,style=numeric,backref=true,natbib=true]{biblatex}
\DefineBibliographyStrings{english}{%
  backrefpage = {page},
  backrefpages = {pages},
}
\colorlet{notegreen}{green!50!black}

\theoremstyle{definition}
\newtheorem{theorem}{Theorem}
\newtheorem{definition}[theorem]{Definition}
\newtheorem{lemma}[theorem]{Lemma}

\newtheorem{example}[theorem]{Example}

\newtheorem*{theorem*}{Theorem}
\newtheorem*{definition*}{Definition}

\newtheorem*{question*}{Question}

\newcommand\argmin{\operatorname*{argmin}}

\DeclarePairedDelimiter{\paren}{\lparen}{\rparen}
\DeclarePairedDelimiterX{\twobarparen}[2]{(}{)}{#1\;\delimsize\|\;#2}

\newcommand{\grad}{\nabla}
\renewcommand{\vec}[1]{\bm{#1}}

\renewcommand{\P}[1]{\mathbb{P}\paren{#1}}

\newcommand{\1}[1]{\mathds{1}\left\{#1\right\}}
\newcommand{\abs}[1]{\left\vert{} #1 \right\vert{}}

\renewcommand{\vec}[1]{\bm{#1}}

\newcommand{\R}{\mathbb{R}}

\newcommand{\norm}[1]{\left\lVert#1\right\rVert}

\DeclareSymbolFont{extraup}{U}{zavm}{m}{n}
\DeclareMathSymbol{\varheart}{\mathalpha}{extraup}{86}
\DeclareMathSymbol{\vardiamond}{\mathalpha}{extraup}{87}

\DeclareUnicodeCharacter{2212}{$-$}

\newtheorem*{claim*}{Claim}

\newcommand{\lmin}{\lambda_{\text{min}}^{\scaleobj{0.75}{\ne0}}}
\newcommand{\lmax}{\lambda_{\text{max}}}
\newcommand{\I}{\mathbb{I}}
\newcommand{\dpw}{\delta_{\text{PW}}}
\newcommand{\sest}{\hat{\vec{s}}}
\newcommand{\f}[1]{f_{\vec{#1}}(\vec{x})}

\title{Lottery Tickets in Linear Models: An Analysis of Iterative Magnitude Pruning}

\author{%
    Bryn Elesedy\thanks{Correspondence to \href{mailto:bryn@robots.ox.ac.uk}{\texttt{bryn@robots.ox.ac.uk}.}}\\
    Department of Computer Science\\
    University of Oxford
    \And
    Varun Kanade\\
    Department of Computer Science\\
    University of Oxford
    \And
    Yee Whye Teh\\
    Department of Statistics\\
    University of Oxford
}
\begin{document}

\maketitle

\begin{abstract}
    \noindent We analyse the pruning procedure behind the lottery ticket
    hypothesis~\citet{frankle2018lottery},
    \emph{iterative magnitude pruning} (IMP), when applied to linear models trained by
    gradient flow.
    We begin by presenting sufficient conditions on the statistical structure of
    the features under which IMP prunes those features that have smallest projection onto the data.
    Following this, we explore IMP as a method for sparse estimation.
    %We hope that our work will contribute to a theoretically grounded understanding
    %of lottery tickets and how they emerge from IMP.
\end{abstract}

\section{Introduction}
The lottery ticket hypothesis~\cite{frankle2018lottery} asserts that a randomly
initialised, densely connected feed-forward neural network contains
a sparse sub-network that, when trained in isolation,
attains equal or higher accuracy than the full network.
These sub-networks are called \emph{lottery tickets} and
the method used to find them is iterative magnitude pruning (IMP).
A network is given a random initialisation, trained by some form of gradient descent for a
specified number of
iterations and a proportion of its smallest weights (by absolute
magnitude) are deleted. The remaining weights are then reset to their
initialised values and the network is retrained. This procedure can be
performed multiple times, resulting in a sequence of sparse yet trainable
sub-networks.\footnote{These sub-networks are \emph{weight} sub-networks, 
    formed by setting to 0 entries of the weight matrices (edges in the graph).
    This is in contrast to the neuron pruning, which removes entire neurons 
    (nodes in the graph). All mentions of sub-network in this paper refer to 
weight sub-networks.}

This simple procedure gives quite surprising results. 
The sub-networks uncovered by IMP are trainable from their original initialisation
and achieve accuracies comparable with, and often better than, the full network.
The same sub-networks perform poorly when reinitialised~\cite{frankle2018lottery}.
Moreover, as observed in~\cite{zhou2019deconstructing},
many of the sub-networks found by IMP have better than random test loss at initialisation,
suggesting that IMP has some ability to select good inductive biases for a problem.
These results point to avenues for better computational and memory efficiency in neural networks,
as well as to properties of deep networks and their training dynamics
that we do not fully understand.
In particular, it is not yet known why weight magnitude provides a good signal
on which to base a pruning heuristic in neural networks.
In other words, why is IMP effective in neural networks?

Our paper aims to address this question for linear models,
with the hope that this will lay the groundwork for similar study in neural networks.
We make a minor adaptation to the version of IMP used in~\cite{frankle2018lottery},
shown in Algorithm~\ref{alg:IMP},
the difference being that in Algorithm~\ref{alg:IMP} we consider pruning one weight per iteration,
while in~\cite{frankle2018lottery} a proportion of the weights are pruned.
It is straightforward to extend our results to the latter case.
It is similarly easy to modify them to discrete time gradient descent
with appropriate step sizes.

%\paragraph{Contributions}
%We present a theoretical analysis of IMP in the context of linear models.
%We do not argue that IMP should be used as a method of sparse estimation in linear models.
%Instead, we aim to initiate theoretical study into this aspect of lottery tickets.
%To the best of our knowledge, prior to this work there has been no
%theoretical study of iterative magnitude pruning or the role it 
%plays in the lottery ticket phenomenon.

%\paragraph{Limitations}
%The linear model is a
%natural first step for understanding IMP in neural networks.
%However, this simplification is the principal limitation of our work.
%Pruning a weight in a linear model is equivalent to removing a feature.
%On the other hand, there is no clear notion of a feature in a neural network and
%the separation of roles between weights and data representations is less distinct.
%We hope that this will be addressed in further work.

\begin{algorithm}
  \DontPrintSemicolon
  \KwIn{
  Loss function $L: \R{}^p \to \R$,
  training time $T \in \R{}_+$,
  initialisation $\vec{w}^{\text{init}} \in \R{}^p$,
  iterations of pruning $q < p$.
  }
  \KwOut{$\vec{w}^{(q)}(T)$}
  Set $M = \I_p$\;
  \For{$k = 0$ to q}{
      Initialise $\vec{w}^{(k)}(0) = M\vec{w}^{\text{init}}$\;
      Train $\dot{\vec{w}}^{(k)}(t) = - M \grad L(\vec{w}^{(k)}(t))$ for $t\in[0, T]$\;
      Set $i = \argmin_{j \in [p]} \left\{\abs{w^{(k)}_j(T)} : M_{jj} = 1\right\}$\;
      Set $M_{ii} = 0$\;
  }
  \Return $\vec{w}^{(q)}(T)$\;
  \caption{Iterative Magnitude Pruning}\label{alg:IMP}
\end{algorithm}

\subsection{Related Work}\label{sec:related-work}
The lottery ticket hypothesis~\cite{frankle2018lottery} has sparked a lot of interest. 
Already, there are more follow up works than it is possible to cover in this article.
To name a few: more reliable discovery of lottery tickets in deep networks with rewinding~\cite{franklestabilizing},
generalising lottery tickets across tasks~\cite{morcos2019one},
constructing networks that perform well with random 
weights~\cite{zhou2019deconstructing,gaier2019weight},
pruning at initialisation~\cite{hayou2020robust}
and finding lottery tickets in randomly weighted networks~\cite{ramanujan2020hidden}. 
For an empirical comparison of various pruning methods see~\cite{blalock2020state}.

Theoretical work has mainly focussed on lottery tickets at at initialisation.
\citet{malach2020proving}
showed that any continuous function
can be approximated by a sub-network of a sufficiently large, randomly weighted neural network.
Following this, there is a body of work on the approximation
properties of lottery tickets at initialisation~\cite{orseau2020logarithmic,pensia2020optimal}.
The theory relating to IMP is relatively undeveloped, which is a key motivation for this paper.

We note some similarities between IMP and methods in compressive sensing,
which is a subfield of signal processing that attempts to find sparse reconstructions 
of noisy signals.
For a systematic review see~\cite{rani2018systematic} and
for a comprehensive discussion of thresholding and similar denoising techniques 
see~\cite[Chapter 11]{mallat2008wavelet}.
Most similar to IMP are two well known thresholding methods that we discuss below,
each of which attempts to recover a sparse
signal $\vec{s} \in \R{}^p$ from noise corrupted measurements $\vec{y}\in \R{}^n$.

Hard thresholding, studied in
the wavelet basis by~\citet{donoho1994ideal}, performs least-squares estimation
$\sest$ of $\vec{s}$ and then applies elementwise to
$\sest$ the thresholding operator $H_\tau(z) = \1{\abs{z} > \tau}z$.
Iterative hard thresholding~\cite{blumensath2008iterative}
consists of linear projection of $\vec{s}$ onto a feature 
matrix $\Psi \in \R{}^{n \times p}$ and then iteratively solving
the linear system $\vec{y} = \Psi \vec{s}$, thresholding the solution at each stage. 
Specifically, iterative hard thresholding estimates $\vec{s}$ by $\sest$ using the update rule
$
    \sest_i^n = H_\tau(\sest^{n}_i + \eta\Psi^\top (\vec{y} - \Psi \sest^{n}_i))
    $
with initialisation chosen by the user and where $\eta$ is a step size.
The argument of $H_\tau$
is exactly the update from gradient flow on the squared error loss 
$\frac12 \norm{\vec{y} - \Psi \sest}_2^2$.
IMP can therefore be viewed as a variation
of iterative hard thresholding, where the threshold operator
is replaced by the restriction to a subset
of indices that is chosen at time $T$ of each training run.

\subsection{Setup and Notation}\label{sec:notation}
We write $[n]$ for the set $\{1, \dots, n\}$.
Let the training inputs be
$\vec{x}_1,\, \dots,\, \vec{x}_n$ and targets be $y_1, \, \dots,\, y_n\in \R$. 
For the parameters we write $\vec{w} \in \R{}^p$ and the 
features are written $\phi_i(\vec{x})$ for $i \in [p]$.
In this work we consider the training data and features to be deterministic,
but it is possible to extend our main result to a random design setting.
Let $X \in
\R{}^{n\times d}$ and $\vec{y} \in \R{}^n$ be the usual row stacking of the
training examples and write $\Phi \in \R{}^{n\times p}$ for the row stacking of
the features. For each $i \in [p]$, write $\phi_i(X) \in \R{}^{n}$ for the
vector with components $(\phi_i(\vec{x}_1), \, \dots , \,
\phi_i(\vec{x}_n))^\top$.
A linear model is any predictor $\f{\theta} = \vec{\theta}^\top \vec{\phi}(\vec{x})$
that is linear in the features, where $\vec{\theta} \in \R{}^p$ are the learned parameters.
The features can be arbitrary non-linear functions, for instance, 
the outputs of the penultimate layer of a pre-trained neural network (with earlier weights fixed) or 
gradients of a neural network at initialisation as with linearised networks~\cite{lee2019wide}.
Write $\Sigma = \frac1n \Phi^\top \Phi$ for the
empirical covariance matrix and notice that $\Sigma_{ij} = \frac1n \phi_i(X)^\top \phi_j(X)$.
We write $\Sigma^+$ for the Moore-Penrose pseudo-inverse of $\Sigma$ and
$\lmax = \norm{\Sigma}_2$ for its operator norm.\footnote{%
        We define the Moore-Penrose pseudo-inverse of $\Sigma$ as the
        matrix that is diagonal in any basis in which $\Sigma$ is diagonal
        and has eigenvalues $\gamma_i$ $i=1,\dots, p$ where $\gamma_i = 0$ if 
        $\lambda_i = 0$ and $\gamma_i = 1/\lambda_i$ otherwise, where 
        $\lambda_i$ $i=1,\dots, p$ are the eigenvalues of $\Sigma$.%
    }
For any matrix $A$, let $\lmin(A)$ denote the smallest non-zero eigenvalue of $A$.
Recall that at each iteration of Algorithm~\ref{alg:IMP} the
diagonal matrix $M \in \R{}^{p\times p}$ records the weights pruned so far, that is $M_{ij} =
\1{i = j \, \land \, w_i \text{ not yet pruned}}$.
Pruning can be seen as either removing features entirely (so reducing the dimension of
$\Phi$) or setting to 0 the corresponding weights or column of $\Phi$ (but preserving all the dimensions).
For the work in this paper these perspectives are equivalent and we will use 
each description interchangeably.

\section{Warm Up: Pruning Heuristic of IMP}\label{sec:warm-up}
In this section we give an analysis to demonstrate that, given some statistical
assumptions on the features, IMP preferentially prunes the features that explain
the data the least in terms of linear projections.
We call this pruning heuristic the \emph{alignment heuristic}.
Specifically, the alignment heuristic prunes $w_i$ 
where $i = \argmin_{j}\{\abs{\phi_j(X)^\top \vec{y}}: M_{jj} = 1\}$.
We consider examples of various $\Sigma$ and examine the pruning heuristic that
arises. For now we assume that $\Sigma$ is full rank, so training
to convergence with $L(\vec{w})=\frac{1}{2n}\norm{\Phi \vec{w} - \vec{y}}_2^2$ 
gives
$
    \vec{w}^{(i)}(\infty) = \frac1n \Upsilon_{i}^{-1} P_{i}^\top \vec{y}
$
where $\Upsilon_i$ and $P_i$ are, respectively, the restrictions of $\Sigma$ and $\Phi$ to the parameters
not yet pruned at iteration $i$ of Algorithm~\ref{alg:IMP}.
$\Upsilon_i$ will always be invertible
as long as $\Sigma$ is invertible, 
which follows from the Cauchy interlace
theorem and $\Sigma$ being positive definite.
We will relax our
notation back to $\Sigma$ and $\Phi$ for the following examples, with the
understanding that each of these results apply
for each iteration of Algorithm~\ref{alg:IMP}, restricting the matrices accordingly.
In particular, it should be noted that in the following examples the conditions
on $\Sigma$ transfer to conditions on $\Upsilon_i$ and $P_i$. Finally, for the rest of this
section we normalise the features so that $\Sigma_{ii} = 1$ $\forall i \in [p]$.

\begin{example}[$\Sigma = \I$]\label{example:identity}
    Gradient flow converges to $n \, w_i(\infty) = \Sigma^{-1} \Phi^\top \vec{y}= \phi_i(X)^\top \vec{y}$.
    According to Algorithm~\ref{alg:IMP} the weight then pruned is $w_i$ where
    $i = \argmin_{j \in [p]} \{\vert{}w^{(i)}_j(\infty)\vert{} : M_{jj} = 1\}$.
    We see immediately that this is equivalent to the alignment heuristic.
\end{example}

\begin{example}[Uniform correlations]\label{example:uniform}
    We express 
    \[
        \Sigma = \I + \alpha (\vec{1}\vec{1}^\top - \I)
    \]
    for $\alpha \in (0, 1)$,
    where $\vec{1} = {(1, 1, \dots, 1)}^\top \in \R^p$.
    We can use the Sherman-Morrison formula to calculate
    \[
        \Sigma^{-1} = \frac{1}{1 - \alpha}\I - \frac{\alpha(1 - \alpha)}{1 + \alpha (p-1)}\vec{1}\vec{1}^\top
    \]
    and we see that
    \[
            n \, w_i(\infty) 
            = \frac{1}{1 - \alpha}\phi_i(X)^\top\vec{y} -
            \frac{\alpha(1 - \alpha)}{1 + \alpha (p-1)}\sum_{j}\phi_j(X)^\top\vec{y}.
    \]
    We conclude that, if $\alpha \ll 1$ is small enough, then Algorithm~\ref{alg:IMP}
    will prune according to the alignment heuristic.
    The case $\alpha \approx 1$ gives this outcome too, but in this case $\Sigma$ 
    is barely invertible. Intuitively, in the case $\alpha \approx 1$ the features are similar
    and pruning one is as good as pruning another.
\end{example}

\begin{example}[Pairwise incoherence]\label{example:pw}
    We say that $\Phi$ satisfies the pairwise incoherence assumption with parameter $\dpw$ if
    \[
        \dpw = \max_{i, j}\abs{\frac1n (\Phi^\top\Phi)_{ij} - \1{i = j}}.
    \]
    Assume that this is the case. Write $\Sigma = \I - A$, then formally we have 
    \[
        \Sigma^{-1} = (\I - A)^{-1} = \sum_{j =0}^\infty A^j.
    \]
    Note that 
    $
    a \coloneqq \norm{A}_2 \leq \sqrt{\norm{A}_1\norm{A}_\infty} \le (p-1)\, \dpw
    $ by H{\"o}lder's inequality, 
    where $\norm{\cdot}_2$ denotes the operator norm when acting on a matrix.
    Taking $\dpw < 1/(p-1)$ therefore ensures that the series converges in operator norm.
    Also, provided the above series converges, we have
    \[
        \norm{\sum_{k=2}^\infty A^k}_2 \leq \frac{a^2}{1 - a}.
    \]
    If $a \ll 1$ we can write $\Sigma^{-1} \approx \I + A$. This gives
    \[
        n\, w_i(\infty) \approx \phi_i(X)^\top \vec{y} + \sum_{j\ne i} A_{ij}\phi_j(X)^\top \vec{y}.
    \]
    The alignment heuristic will be followed if the magnitude of the first term dominates
    that of the second, which is $O(\dpw)$. Therefore, given a suitable pairwise incoherence
    assumption, Algorithm~\ref{alg:IMP} prunes according to the alignment heuristic.
\end{example}

\section{Support Recovery with IMP}\label{sec:support-recovery}
Empirically, IMP has been found to result in sparse yet performant sub-networks.
In this section we explore this phenomenon analytically in the context of linear models,
where the natural application is to sparse estimation.
In our analysis we will make use of a strengthened form of
the restricted nullspace property~\cite{wainwright2019high}.
\begin{definition}[Orthogonal nullspace property]
    Let $S \subset [p]$ be a subset of indices. Define the cone
    \[
        C(S) = \{\vec{x} \in \R{}^p: \norm{\vec{x}_{S^\text{c}}}_1 \le \norm{\vec{x}_{S}}_1\}
    \]
    where $\vec{x}_{Q}$ are the components of $\vec{x}$ with indices in $Q$. A matrix
    $A$ acting on $\R{}^p$
    satisfies the \emph{orthogonal nullspace property} with respect to $S$ if
    all elements of the nullspace of $A$ are orthogonal to all elements of $C(S)$
    \[
        \text{null}(A) \perp C(S).
    \]
\end{definition}
Notice that the cone $C(S)$ is the set of vectors whose 1-norm on the index set
$S$ dominates that on the other indices $S^\text{c}$. In particular, a sparse signal
supported on $S$ will belong to $C(S)$. In the context of Theorem~\ref{thm:imp-estimation}, the above 
condition
ensures that the signal $\vec{s}$ is recoverable.
We now present our sparse estimation result for Algorithm~\ref{alg:IMP}.
\begin{theorem}[Sparse Estimation with IMP]\label{thm:imp-estimation}
    Assume that $\vec{y} = \Phi \vec{s} + \vec{\xi}$ 
    for $\vec{s} \in \R{}^p$ and that $\{\xi_i$: $i=1, \dots, p\}$ are independent, zero mean
    sub-Gaussian random variables
    with variance proxy $\sigma^2$.
    Let $\Sigma \coloneqq \frac1n \Phi^\top \Phi$ and let $\lmin > 0$ 
    be the smallest non-zero eigenvalue of $\Sigma$, which we assume to exist. 
    Suppose that $\vec{s}$ is $k$-sparse and supported on a set 
    $S \subset [p]$ with $\abs{S} = k$, and that
    $\Sigma$ satisfies the orthogonal nullspace property with respect to $S$.
    Let $L(\vec{w})=\frac{1}{2n}\norm{\Phi \vec{w} - \vec{y}}_2^2$ be the mean squared
    error loss.
    Consider running Algorithm~\ref{alg:IMP} with $L$, $\vec{w}^{\text{init}} = \vec{0}$,
    $T=\infty$ and $q \le p - k$, and denote the output by $\vec{v}$. 
    Let $\gamma > 0$, if 
    \[  
        n \geq \frac{8\sigma^2}{\gamma^2 \lmin}\log(2p/\delta)
    \]
    then, with probability at least $1 - \delta$ we have 
    \begin{enumerate}[(i)]
        \item $\vec{v}$ is at least $(p-q)$-sparse. 
        \item No false exclusion above $\gamma$: $v_i \neq 0$ for any $i$
            with $\abs{s_i} \ge \gamma$.
    \end{enumerate}
\end{theorem}
\begin{proof}
    Point (i) is obvious. We establish point (ii).
    Consider the first iteration of Algorithm~\ref{alg:IMP}.
    Let $V$ be the column space of $\Sigma$. Since $\Sigma$ satisfies the orthogonal nullspace
    condition
    with respect to the support of $\vec{s}$, we know that the projection of $\vec{s}$ on to 
    $V$ is $\vec{s}$. 
    By diagonalising, we can see that $\Sigma^{+} \Sigma$ is
    exactly this projection. Training until convergence therefore gives
    \[
        w_i(\infty) = s_i + \frac1n (\Sigma^{+} \Phi^\top \vec{\xi})_i.
    \]
    Then on this iteration we can be sure not to prune any $i$ with $\abs{s_i} \ge \gamma$
    if
    \[
        \frac1n \abs{\left(\Sigma^{+} \Phi^\top \vec{\xi}\right)_i} < \gamma / 2 \quad 
        \forall \, i.
    \]
    By Lemma~\ref{lemma:noise-control}, this happens with probability at least
    $1 - \delta$ if 
    $
        n \geq \frac{8\sigma^2}{\gamma^2 \lmin}\log(2p/\delta).
    $
    To demonstrate (ii), we show that this is sufficient to guarantee (with high probability)
    that
    no $i$ with $\abs{s_i} \ge \gamma$ is pruned on any iteration
    of Algorithm~\ref{alg:IMP}. We conclude the proof with this argument.

    Let $\Upsilon$ be the 
    sub-matrix of $\Sigma$ formed by removing the rows and columns with indices
    $i_1, \dots, i_k$ $k < p$, let $P$ be the sub-matrix of $\Phi$ by removing
    the same columns. The orthogonal nullspace condition on $\Sigma$
    means that the restriction of $\vec{s}$ has no
    component in the null of $\Upsilon$ at any iteration. 
    Hence, we need only show that
    \[
        \frac1n\abs{(\Upsilon^{+} P^\top \vec{\xi})_i} < \gamma / 2.
    \]
    Again by Lemma~\ref{lemma:noise-control}, this happens with probability at least
    $1 - \delta$ if 
    $
    n \geq \frac{8\sigma^2}{\gamma^2 \lmin(\Upsilon)}\log(2p/\delta).
    $
    $\Sigma$ is symmetric, so we may apply Cauchy's interlace theorem 
    (see~\cite{horn2012matrix}, Theorem 4.3.17) to obtain
    $\lmin(\Upsilon) \ge  \lmin(\Sigma)$. The proof is complete.
\end{proof}
This result tells us that if we have $n = O(\gamma^{-2} \log(2p/\delta))$ samples
then, with high probability, IMP can recover the support of a sparse signal
wherever it has magnitude at least twice the noise level $\nicefrac{\gamma}{2}$.
It is not possible to place guarantees on
recovery of components with magnitude smaller than $\gamma$ without
assumptions on the relative sizes of the components of $\vec{s}$.
It should be straightforward to modify~\cref{thm:imp-estimation}
to account for random features by a bound on $\lmin$.

\section{Discussion}
We have shown that IMP can recover the support of a sparse signal under mild
assumptions on the design matrix and we gave bounds for the estimation
error in this setting. 
Further work may seek to extend the results of this paper to 
neural networks, a tractable route for which
might be found using linearised networks~\cite{lee2019wide}.
Alternatively, it may be interesting to consider random features.

\begin{ack}
    We thank Sheheryar Zaidi and Bobby He for helpful feedback on early versions of this work.
    We would also like to thank Zhanyu Wang for pointing out two errors in a previous version of this work.
    BE is supported by the UK EPSRC CDT 
    in Autonomous Intelligent Machines and Systems (grant reference EP/L015897/1).
    VK is supported
    in part by the Alan Turing Institute under the EPSRC grant EP/N510129/1.
\end{ack}

\appendix

\section{Sub-Gaussian Concentration}\label{sec:technical-results}
\begin{lemma}\label{lemma:noise-control}
    Define $\Phi \in \R{}^{n \times p}$ and $\Sigma \in \R{}^{p\times p}$ with $\Sigma^{+}$ 
    the Moore-Penrose pseudo-inverse of $\Sigma$ and $\lmin$ the non-zero eigenvalue of
    $\Sigma$ that is smallest in absolute magnitude (which we assume exists).
    Let $\vec{\xi} \in \R{}^n$ be a vector with elements $\xi_i$ that are
    iid sub-Gaussian with zero mean and variance proxy $\sigma^2$. If
    $
        n \ge \frac{2 \sigma^2}{\epsilon^2 \lmin} \log (2p / \delta)
        $
    then with probability at least $1- \delta$ we have 
    $
        \max_j \abs{\frac1n ( \Sigma^{+} \Phi^\top \vec{\xi})_j} < \epsilon
        $.
\end{lemma}
\begin{proof}
    Define $\vec{\alpha} = \frac1n \Sigma^{+} \Phi^\top \vec{\xi}$. Let 
    $A = \frac1n \Sigma^{+} \Phi^\top$, then it is straightforward to check that,
    for any $i$,
    $\alpha_i = \sum_{j} A_{ij}\xi_j$ is sub-Gaussian
    %\[
        %\E{\e^{s \alpha_i}}  = \prod_{j} \E{\e^{s A_{ij} \xi_j}} 
        %\le \prod_{j} \e^{\frac{\sigma^2 s^2}{2} A_{ij}^2} = \e^{\frac{\sigma^2 s^2}{2}\sum_j A_{ij}^2} 
    %\]
    with variance proxy $\sigma^2  \sum_{j} A_{ij}^2$. In addition to this, we have
    \[
            \sum_{j} A_{ij}^2 
                = (AA^\top)_{ii} = \frac1n \Sigma^{+}_{ii}
                \le \frac1n \max_{ij}\abs{\Sigma^{+}_{ij}} \le \frac1n \norm{\Sigma^{+}}_2 = \frac{1}{n \lmin},
    \]
    %\begin{align*}
        %\sum_{j} A_{ij}^2 
            %&= (AA^\top)_{ii} = \frac1n \Sigma^{+}_{ii} \\
            %&\le \frac1n \max_{ij}\abs{\Sigma^{+}_{ij}} \le \frac1n \norm{\Sigma^{+}}_2 = \frac{1}{n \lmin},
    %\end{align*}
    %where the final inequality uses Lemma~\ref{lemma:sym-max-op-norm}.
    So the standard tail bound gives, for any $i$,
    $
        \P{\abs{\alpha_i} > \epsilon} \le 2\exp \left(-\frac{n \lmin\epsilon^2}{2\sigma^2}\right)
        $
    and the conclusion follows from a union bound.
\end{proof}

\printbibliography
\end{document}